\documentclass[letterpaper, 10 pt, conference]{ieeeconf}  

\IEEEoverridecommandlockouts                              

\usepackage[bookmarks=true]{hyperref}
\usepackage{color}
\usepackage{cite}
\usepackage{notoccite}

\usepackage{graphicx}
\def\argmin{\mathop{\arg\min}}	%
\usepackage{bm}
\usepackage{tabularx}
\usepackage{floatrow}
\usepackage[utf8]{inputenc}
\usepackage{amsmath}
\usepackage[utf8]{inputenc} 
\usepackage[T1]{fontenc}    
\usepackage[mathscr]{eucal}
\usepackage{caption}
\usepackage{amsfonts}
\usepackage{dblfloatfix} 
\usepackage{epsfig}
\usepackage{makecell}
\usepackage{multirow}
\usepackage{hyperref}
\usepackage{dsfont}
\usepackage{breqn}
\usepackage{amssymb}
\usepackage{amsmath,amsthm}
\usepackage{amsfonts}       
\usepackage{nicefrac}       
\newtheorem{theorem}{Theorem}
\usepackage{amsthm}
\usepackage{bbm}
\newtheorem{problem}{Problem}

\input{mysymbol.sty}

\usepackage{mathtools,algorithm}
\usepackage[noend]{algpseudocode}
\title{\LARGE \bf Graph Neural Networks for Motion Planning
}

\author{Arbaaz Khan$^{1,2}$, Alejandro Ribeiro$^{1}$, Vijay Kumar$^{1}$, Anthony Francis$^{2}$ 
\thanks{$^{1}$GRASP Lab, University of Pennsylvania, $^{2}$Google Brain. 
        {\tt\small arbaazk@seas.upenn.edu}}%
}

\begin{document}

\maketitle
\thispagestyle{empty}
\pagestyle{empty}

\begin{abstract}
This paper investigates the feasibility of using Graph Neural Networks (GNNs) for classical motion planning problems.
We propose guiding both continuous and discrete planning algorithms using GNNs' ability to robustly encode the topology of the planning space using a property called permutation invariance.
We present two techniques, GNNs over dense fixed graphs for low-dimensional problems and sampling-based GNNs for high-dimensional problems.
We examine the ability of a GNN to tackle planning problems 
such as identifying critical nodes or 
learning the sampling distribution in Rapidly-exploring Random Trees (RRT). 
Experiments with critical sampling, a pendulum and a six DoF robot arm show
GNNs improve on traditional analytic methods as well as learning approaches using fully-connected or convolutional neural networks. 
\end{abstract}

\section{Introduction}
\label{sec:Sec1}

 Motion planning is a widely studied problem with applications in robotics, computer graphics and medicine \cite{latombe1999motion}. Early planning methods such as Djikstra's search a discretized version of space, but the states required explode exponentially as the dimensionality of the space increases. $\text{A}^*$ looks to improve upon Djikstra's by informing the search with a heuristic \cite{russell2002artificial}, but heuristic design is a challenging problem \cite{bonet2001planning}. 
Further, in scenarios such as robotics, naive discretization can violate kinodynamic constraints. Sampling-based Planners (SBPs) such as Rapidly-exploring Random Trees (RRTs) tackle these issues by approximating the
topology of the configuration space (C-space), i.e the space of
all possible agent configurations. These methods sample
points in C-space and connect these points in a graph or tree if a collision free trajectory is feasible \cite{lavalle2006planning}. 

One caveat of SBPs is that the samples needed to cover a space increase with C-space dimensionality \cite{hsu1997path}.
Researchers have tackled this issue using deep learning to identify samples in C-space which are in some sense more important.
For example, for a robot navigating an office, samples in narrow corridors are more important than in free space. 
Ichter et al. \cite{ichter2018learning} use the latent space of a conditional variational autoencoder (CVAE) to bias SBP sampling towards critical samples. RL-RRT \cite{chiang2019rl} use deep reinforcement learning (RL) to bias tree-growth towards promising regions of the C-space. \cite{zhang2018learning} use RL to learn an implicit sampling distribution to reduce samples required. Critical PRMs \cite{ichter2019learned} directly learn these critical samples while LEGO \cite{kumar2019lego} learn critical samples for graph search algorithms such as $\text{A}^*$. 

A common theme of the above work is the use of convolutional neural networks (CNNs) or fully-connected networks (FCNs) to learn about the planning space. In this paper, we argue  most planning spaces have rich topological structure which may not lie on a two dimensional lattice. When using a CNN or FCN, information about this structure is lost, and we show results that demonstrate significant changes to the underlying structure of the environment can result in poor generalization of methods that use CNNs and FCNs. 
\begin{figure}[t!]
  \centering
  \includegraphics[width=\linewidth]{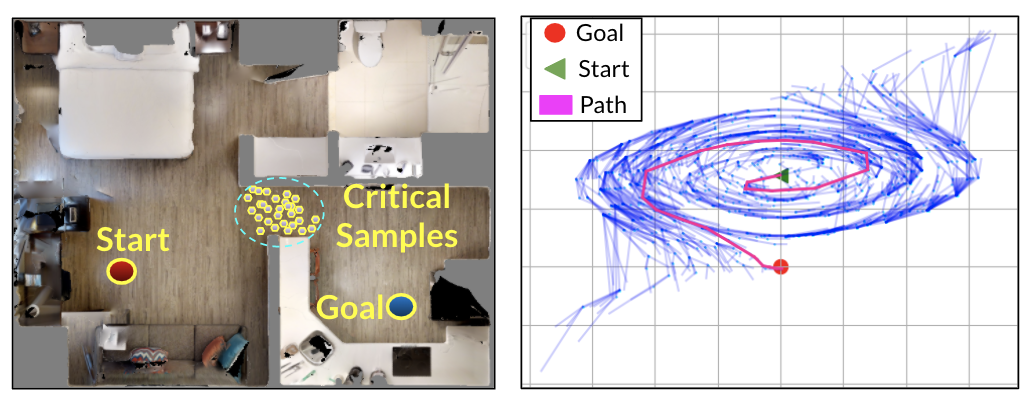}
  \caption{\textbf{Graph Neural Networks for Motion Planning.} \textbf{a)} Identifying critical samples for a given start and goal. \textbf{b)} GNN-based sampler guides RRT tree-growth for a pendulum.
  \label{fig:mainfig}}
\end{figure}

Graph Neural Networks (GNNs) have
successfully tackled problems with rich graphical structure
such as text classification \cite{kipf2016semi}, protein interface prediction \cite{fout2017protein}, parsing social relationships \cite{wang2018deep}, and large-scale multi-agent reinforcement learning \cite{khan2019graph}, \cite{khan2019graph1}. For planning problems where C-space topology is an irregular graph, we hypothesize GNNs offer better solutions than existing learning methods or naive heuristics because of a property called \textit{permutation invariance}. Gama et al. \cite{gama2019stability} have shown that GNNs are invariant to graph permutations; in this paper, we show permutation invariance captures an equivalence class of motion planning problems. We hypothesize this will enable GNN-based planners to transfer well to new unseen environments. This paper proposes GNN solutions to these motion planning problems:\\
\textbf{Identifying Critical Samples} Given a static dense graph that covers the C-space of interest, a GNN is trained to identify critical nodes relevant to a given planning problem. The quality of critical samples computed by the GNN architectures is compared against those produced by CNNs, or by sampling the latent space of a (CVAE) as in \cite{kumar2019lego}.\\
     \textbf{Sampling with Graph Neural Networks} A GNN is used to compute the parameters of the sampling distribution from which the samples are drawn to expand the graph in a SBP. We show a GNN based sampler guides the search tree towards the goal node faster than a uniform sampler. We also show the GNN sampler can handle higher-dimensional problems by planning for a robot arm with six degrees of freedom.  
We experimentally show GNNs outperform existing solutions for machine learning in motion planning. 

%
\section{Preliminaries}
Here we outline notation for planning problems along with a brief overview of GNNs. 
Let $\mathcal{X}\in\mathbb{R}^d$ represent a $d$-dimensional C-space. $\mathcal{X}_{\text{obs}}$ denotes the part of the space occupied by obstacles, and free space is $\mathcal{X}_{\text{free}}=\mathcal{X}\setminus \mathcal{X}_{\text{obs}}$. Let the initial condition be $x_{\text{init}}
\in \mathcal{X}_{\text{free}}$ 
and the goal condition given as $x_g \in \mathcal{X}_{\text{free}}$ 
The path $\mathbf{p}$ is said to be feasible if it is collision free $\mathbf{p}(\zeta) \in \ccalX_{\text{free}}$ $\forall$ $\zeta \in [0,1]$ and if  $\mathbf{p}(0)=x_{\text{init}}$,  $\mathbf{p}(1)=x_{\text{goal}}$. 
Let there also exist a cost function $c(p)$ that maps the path $\mathbf{p}$ to a bounded cost $[0,c_{\text{max}}]$.
A motion planning problem $\bbM$ is represented as the tuple, $\bbM=$ $\{\mathcal{X}_{\text{free}},x_{\text{init}},x_{\text{goal}}\}$
Later, we build on this notion of the planning problem to achieve different objectives. 

\subsection{Graph Neural Networks}
GNNs can be seen as generalizations of CNNs to general graphs by employing convolutional graph filters \cite{gama2018convolutional}. Consider a graph $\mathcal{G}=(\mathbf{V},\mathbf{E})$ described by a set of $N$ nodes denoted as $\mathbf{V}$, and a set of edges denoted by $\mathbf{E} \subseteq \mathbf{V} \times \mathbf{V}$. Let this graph act as support for data $\mathbf{x}=[\mathbf{x}_1,\ldots,\mathbf{x}_{N}]\in\mathbb{R}^{N \times m}$. The relationship between $\mathbf{x}$ and $\mathcal{G}$ can be completely characterized by a matrix $\mathbf{S}$ called the graph shift operator. The elements of $\mathbf{S}$ given as $s_{ij}$ respect the sparsity of the graph, i.e $s_{ij} = 0$, $\forall$ $i\neq j \text{ and } (i,j) \notin \mathbf{E}$. Examples for $\mathbf{S}$ are the adjacency matrix, the graph laplacian, and the random walk matrix. $\mathbf{S}$ can be used to define the map $\mathbf{y} = \mathbf{S}\mathbf{x}$. If the set of neighbors of node $n$ is given by $\mathfrak{B}_n$ then 
the operation $[\mathbf{S}\mathbf{x}]_n =\sum_{j=n,j\in \mathfrak{B}_n}s_{nj}\mathbf{x}_n$ performs a simple aggregation of data at node $n$ from its neighbors that are one hop away. Recursively, one can access information from nodes located further away. For example, $\mathbf{S}^k\mathbf{x} = \mathbf{S}(\mathbf{S}^{k-1}\mathbf{x})$ aggregates information at each node from its $k$-hop neighbors. Using this map, the spectral $K$-localized graph convolution is defined as:
\begin{equation}\label{eq:z}
    \mathbf{z} = \sum_{k=0}^{K} h_k \mathbf{S}^k \mathbf{x} = \mathbf{H(S)x}
\end{equation}
where $\mathbf{H(S)} = \sum_{k=0}^{\infty} h_k \mathbf{S}^k$ is a linear shift invariant graph filter \cite{segarra2017optimal} with coefficients $h_k$. In practice, the output of a graph convolutional filter is followed by a nonlinearity $\sigma$ to produce $\bby$. Thus, the output at the first layer is given as:
\begin{equation}\label{eqn_layer_1}
   \bby_1 = \sigma\Big[\ \bbz_1 \ \Big] 
             = \sigma\bigg[ \sum_{k=0}^{K} h_{1k} \mathbf{S}^k \mathbf{x}\bigg].
\end{equation}
where $h_{1k}$ are the filter coefficients of the first layer. 
In general, there are a total of $L$ layers each of which is  which produces output the $\bby_L$ according to the recursion 
\begin{equation}\label{eqn_gnn_recursion}
   \bby_L = \sigma\Big[\ \bbz_L \ \Big] 
             = \sigma\bigg[  \sum_{k=0}^{K} h_{Lk} \mathbf{S}^k \mathbf{y}_{L-1} \bigg].
\end{equation}
Eqns \ref{eqn_layer_1} and \ref{eqn_gnn_recursion} outline the main graph convolution operations. Several variants of GNNs build on these basic graph convolutions, including the Graph Attention Transformers (GATs) \cite{velivckovic2017graph} we employ in this paper.

%

\section{Learning Critical Samples with GNNs}
\label{sec:learningcriticalsamples}
Consider a graph $\ccalG=(\bbV,\bbE,\bbW)$ covering the $d$-dimensional configuration space $\ccalX \in \mathbb{R}^d$. The set of vertices $\bbV$ represents a collection of points in $\mathbb{R}^d$, $\bbE$ represent edges between these points and
edge weights between node $u$ and $v$ given by $w_{uv} \in \bbW$ is the cost of traversing the edge. Each node $n$ is equipped with a feature representation $\mathbf{x}_n \in \mathbb{R}^m$. For example, in a two-dimensional planning problem, vertices are randomly sampled points in the space with edge connections based on some $k$-nearest neighbor rule and  node $n$'s features can consist of its own $xy$ position, $x_{\text{init}}$ and $x_{\text{goal}}$. Feature choice is arbitrary and can be adapted to suit the problem. For the full graph, the feature vector is given as $\bbx=[\bbx_1,\ldots,\bbx_N]$ where $N$ is the cardinality of the graph. In this section, we are interested in using a graph that covers the configuration space $\mathbb{R}^{d}$ in a uniform manner.  \cite{janson2018deterministic} shows a non-lattice, low dispersion sampling scheme such as a Halton sequence is ideal for uniform coverage, so we choose $\ccalG$ to be a $r$-disc Halton graph. To uniformly cover $\mathbb{R}^d$, $\ccalG$ must be sufficiently large/dense. Note $\ccalG$ remains \textbf{\textit{constant}} even if the planning problem is changed, i.e all planning problems in $\mathbb{R}^d$ can be represented on $\ccalG$.

Given a graph $\ccalG$ and planning problem $\bbM=$ $\{\mathcal{X}_{\text{free}},x_{\text{init}},x_{\text{goal}}\}$, let there exist a graph search algorithm $\bbA(\ccalG,\bbM)$ that finds a path $\mathbf{p}$ on $\ccalG$ that is feasible and has the lowest cost if one exists. However, it is often expensive to run the graph search algorithm $\bbA$ on a dense graph such as $\ccalG$. To overcome this, we hypothesize that the complexity of a given planning problem can be reduced by identifying only those $\hat{\bby}$ nodes in $\mathcal{G}$ that are relevant or \textit{critical}. Critical nodes  can be composed with a sparse graph $\ccalG_s$ which is easier to planon $\bbA$ than the original dense graph $\ccalG$. Defining which nodes are \textit{critical} is a challenging research problem \cite{ichter2019learned}. In this work, we use the Bottleneck Node algorithm proposed in \cite{kumar2019lego} (Alg 1) to generate ground truth critical nodes required for training data due to its guarantees to generate highest cost nodes ("critical") along the shortest path. A drawback of the BN algorithm is that it needs the shortest path as input. These ground truth samples are denoted as $\bby_{g}$ Thus, the critical sample problem considered in this paper is:
\begin{problem}
\label{prob:prob1}
Given motion planning problem $\bbM$ and constant graph $\ccalG$ as support for feature vector $\bbx$, compute critical nodes $\hat{\bby}:=\pi_{\theta}(\ccalG,\bbx,\bbM)$ where $\pi$
is parameterized by $\theta$:
\begin{equation}
   \theta^*= \argmin_{\theta} ||\hat{\bby} - \bby_{g} ||
\end{equation}
\end{problem}

We make one approximation to Problem \ref{prob:prob1}. Instead of predicting critical nodes in the graph, we predict states in the d-dimensional configuration space, $\hat{\bby} \in \ccalX$. This reduces the problem to regression instead of the more complex structured prediction problem over a large graph. A similar strategy for identifying critical nodes is used in \cite{ichter2019learned,kumar2019lego}.

\begin{figure}[bt]
  \centering
  \includegraphics[scale=0.19]{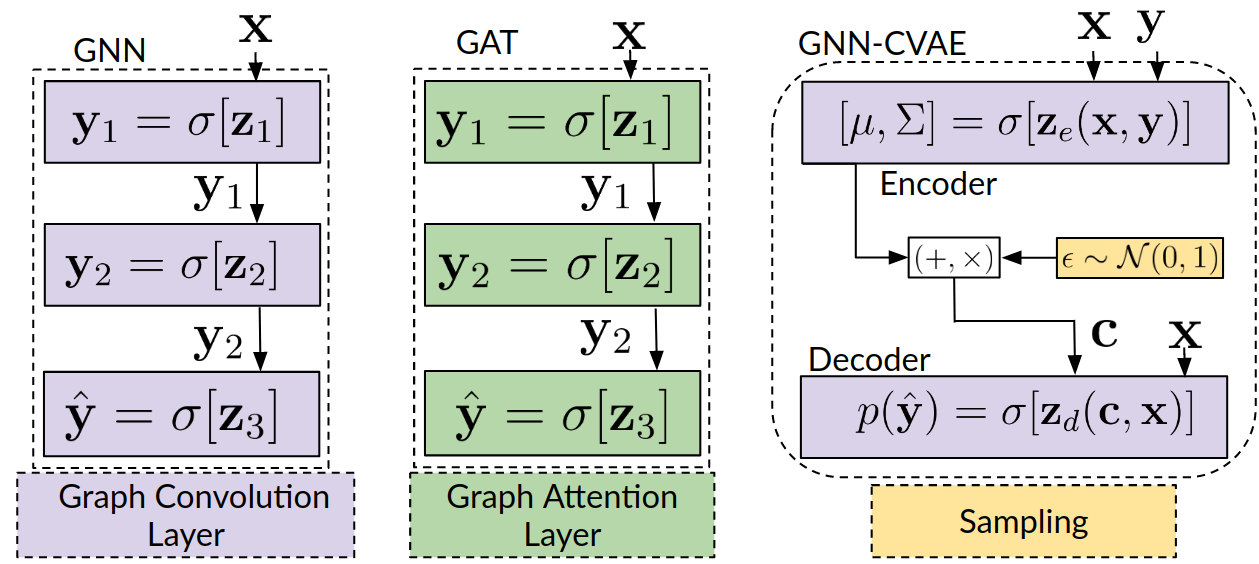}
  \caption{\textbf{GNN Architectures for Motion Planning.} \textbf{a)} GNN models layer graph convolutions; the last layer outputs $\hat{\mathbf{y}}$.  \textbf{b)} GAT models layer attention weighted graph convolutions \cite{velivckovic2017graph} and also output $\hat{\mathbf{y}}$. \textbf{c)} \textbf{GNN-CVAE} models use graph convolutions in the encoder and decoder of a conditional variational auto-encoder and output a distribution for $\hat{\mathbf{y}}$.} 
\label{fig:gnnarch2}
\end{figure}

\subsection{Graph Neural Network Architectures}
\label{subsec:gnnarchs}
We propose three architectures (Fig \ref{fig:gnnarch2}) to study the abilities of GNNs to identify these critical samples. The first one consists of a simple feedforward graph neural network (Fig \ref{fig:gnnarch2}a) and we call this GNN. Here, the input to the GNN is a static graph $\mathcal{G}$ that has enough nodes to cover $\mathbb{R}^d$ and features associated with each of those nodes. The graph convolutional filters aggregate information at each node from their $K$-nearest neighbors as given in Eqn \ref{eqn_layer_1}. Each extra graph convolutional layer provides information from neighbors multiple hops away. For example, in a 2 layer GNN, node $n$ aggregates information from its own neighbors as well as indirect information about its neighbors' neighbors. The GNN architecture also consists of a graph maxpool at the end followed by a FCN layer which predicts one critical state $\hat{\mathbf{y}}$. 

The second architecture investigated is the Graph Attention Transformer or GAT \cite{velivckovic2017graph} (Fig \ref{fig:gnnarch2}b). GATs are similar to GNNs in their use of graph convolutional filters, but instead of aggregating information over all neighbors as in Eqn. \ref{eqn_layer_1}, information from node $n$'s neighbors is scaled by a learned attention weight. We hypothesize that when planning on dense graphs GATs might benefit from attention to focus only on nodes that are critical. Selectively attending to neighbors requires additional parameters to compute attention weights; therefore, GATs can be slower in training and inference. The GAT output $\hat{\mathbf{y}}$ represents a single point in C-space.

Finally, we investigate a modified conditional variational auto-encoder \cite{sohn2015learning} we call GNN-CVAE, consisting of an encoder and decoder. The encoder consists of GNN layers capped by maxpool and FCN layers. The encoder's input is ground truth critical nodes $\mathbf{y}$ and the graph $\mathcal{G}$. The encoder's conditioning variable is the feature vector $\mathbf{x}$ representing parameters of the planning problem at each node. Denoting the latent variable $\tau \in\mathbb{R}^P$, the GNN layers of the encoder map $(\mathbf{y},\mathbf{x})$ to parameters $\phi=[\mu,\Sigma]$ of a Gaussian distribution in latent space given by $q_{\phi}(\tau|\mathbf{x,y},\mathcal{G})$. The target distribution is fixed to the isotropic normal distribution $\mathcal{N}(0,I)$. The decoder has a similar architecture and maps a sample from $\mathcal{N}(0,I)$ conditioned on a planning problem's parameters $\mathbf{x}$ and graph $\mathcal{G}$ to the distribution of the critical nodes $p_{\Theta}(\mathbf{\hat{y}}|\tau,\mathbf{x},\mathcal{G})$ parametrized by $\Theta$. The encoder is used to train the decoder by minimizing the approximate variational lower bound or the Evidence Lower Bound (ELBO):
\begin{equation}
\label{eq:elbo}
    -D_{KL}(q_{\phi}(\tau|\mathbf{x,y},\mathcal{G})||\mathcal{N}(0,I)) + \frac{1}{P} \sum_{i=1}^P \log p_{\Theta}(\hat{\mathbf{y}}|\tau^{(i)},\mathbf{x},\mathcal{G})
\end{equation}
During inference, the decoder predicts a distribution of critical nodes when samples are drawn from $\mathcal{N}(0,I)$; this  improves over the GNN and GAT models by predicting a probability distribution instead of a single point, which is often more desirable. For example, understanding the distribution of free space points inside the narrow corridor in Fig \ref{fig:bnresultsfig} may be critical for real-world planning. We compare these GNN based models against a CNN variant where all the GNN layers are replaced by CNNs and a CVAE which has a mix of CNN/FCN layers depending on the problem.

\begin{figure}[bt]
    \includegraphics[scale = 0.55]{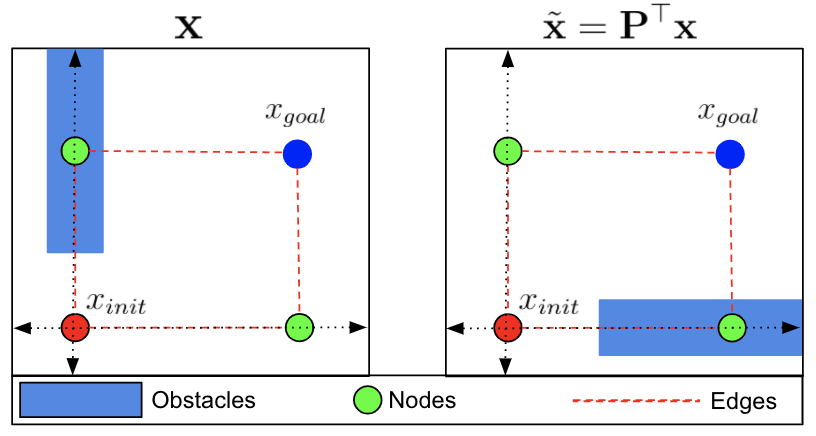}
  \caption{\textbf{Permutation Invariance} A class of motion planning problems can be represented on the same static graph by permuting feature vector $\mathbf{x}.$ }\label{fig:invarfig}
\end{figure}

\subsection{Permutation Invariance}
Here we prove the GNN parametrization for $\pi_{\theta}(\ccalG,\bbx,\bbM)$ for the motion planning problem offers advantages when generalizing to motion planning problems not seen in training. 

Let the adjacency matrix of the graph $\mathcal{G}$ spanning the $d$-dimensional configuration space $\mathbb{R}^d$ be denoted by $\mathbf{S} \in \mathbb{R}^{N \times N}$ Entries of this matrix $S_{nm}$ are binary and are 1 only if the position $x_m \in \mathbb{R}^d$ represented by node $m$ is among the $k$ nearest neighboring nodes to the position $x_n \in \mathbb{R}^d$ represented by $n$,  based on some norm $||.||$ on $\mathbb{R}^d$.
%
%
Further, for a given motion planning problem $\mathbf{M} = \{\mathcal{X}_{free},x_{init},x_{goal} \}$ let the feature vector $\mathbf{x}$  for which the graph $\mathcal{G}$ is the support be defined as:
\begin{equation}
\label{eq:featurestate}
\mathbf{x}=
  \begin{bmatrix}
    \Delta x_{init}^{1} & \Delta x_{g}^{1} & f_1   \\
    \vdots & \vdots & \vdots \\
    \Delta x_{init}^{N} & \Delta x_{g}^{N} & f_N
  \end{bmatrix}
\end{equation}
where $\Delta x_{init}^n = x_{init} - x_n$ and $ \Delta x_{g}^n = x_{goal}-x_n$  and $f_n$ is a feature indicator for node $n$ that takes value zero or one depending on if the node is in free space or inside an obstacle.
We are free to design feature vector $x$ as we wish. 
Permuting the rows of $\bbx$ represents a new motion planning problem where the obstacles have moved, similar to Fig \ref{fig:invarfig}. 
Formally, define a set of permutation matrices of dimension ${N}$ such that $\bm{\mathcal{P}} = \{\mathbf{{P}}\in \{0,1\}^{{N} \times {N}}\    \mathbf{{P1=1}},\mathbf{{P}^\top1 =1}\}$. 
Such a  permutation matrix $\mathbf{P}$ is one for which the product $\mathbf{P}^{\top}\bbx$ reorders the entries of any $\bbx$ and the operation $\mathbf{P}^{\top}\bbS\mathbf{P}$ reorders the rows and columns of any given $\bbS$. Let the permuted feature vector $\tilde{\bbx}= \bbP^\top\bbx$ represent a new motion planning problem $\widetilde{\mathbf{M}}=\{\tilde{\mathcal{X}}_{free},x_{init},x_{goal} \}$ produced by permuting the position of the obstacles but with the same start and goal. 
It is important to note that \textit{not all possible motion planning problems can be represented as a row permutation of $\mathbf{x}$ but nevertheless there exist a large class of motion planning problems that can be represented by permuting the rows of $\mathbf{x}$}. 
Let the optimal GNN filter weights after training on $\bbM$ be given by $\theta^*$ and the optimal filter weights for $\widetilde{\bbM}$ be $\tilde{\theta}^*$. Then:
\begin{theorem}
\label{thm:invariancetheorem}
The optimal GNN filter weights $\theta^*$ for $\bbM$ defined on graph $\mathcal{G}$ with a feature vector $\bbx$ and optimal GNN filter weights $\tilde{\theta}^*$ for $\widetilde{\bbM}$ also defined on graph $\mathcal{G}$ with a feature vector $\tilde{\bbx}$ are equivalent; 
\begin{equation}
    \theta^* \equiv \tilde{\theta}^*
\end{equation}
\end{theorem}
\begin{proof}
We must show the motion planning problem and the GNN parametrizations are permutation equivariant. Proving the motion planning problem is permutation equivariant is trivial since by choice of construction $\tilde{\bbx} = \bbP^\top\bbx$ \footnote{
If we also permute $x_{init}$ and $x_{goal}$, because the feature vector $\bbx$ uses relative rather than absolute positions, $\tilde{\bbx} \approx \bbP^\top \bbx $
}. 

To prove the GNN parametrizations are also permutation equivariant, consider the output of the GNN filter with parameters for the permuted motion planning problem $\widetilde{\bbM}$ with feature vector $\tilde \bbx$ according to Eqn \ref {eqn_layer_1}:
\begin{equation}
    \pi_{\theta}(\ccalG,\tilde{\bbx},\widetilde{\bbM}) 
             = \bigg[ \sum_{k=0}^{K} h_{k} \mathbf{S}^k \tilde{\mathbf{x}}\bigg].
\end{equation}
which can be written as: 
\begin{equation}
\label{eq:permuinvariancebig}
\begin{split}
\bigg[ \sum_{k=0}^{K} h_{k} \mathbf{S}^k \tilde{\mathbf{x}}\bigg] = \bigg[ \sum_{k=0}^{K} h_{k} \mathbf{S}^k \bbP^\top \mathbf{x}\bigg] \\
= \bigg[ \bbP^\top \sum_{k=0}^{K} h_{k} \mathbf{S}^k  \mathbf{x}\bigg] 
 \\
= \bbP^\top  \pi_{\theta}(\ccalG,\bbx,\bbM)
\end{split}
\end{equation}
Intuitively, reordering obstacles in configuration space and reordering rows in $\bbx$ appropriately reorders filter outputs without changing policy weights. Thus, $\theta^* \equiv \tilde{\theta}^*$.
\end{proof}
 
Another consequence of Theorem \ref{thm:invariancetheorem} and Eqn. \ref{eq:permuinvariancebig} is that when $\pi_\theta$ is parametrized by a GNN it offers rotational invariance as well as translation invariance (this can be seen in Fig \ref{fig:invarfig}) as compared to CNNs which only offer translation invariance, helping GNN parametrizations of motion planning problems generalize to obstacles not seen in training better than CNNs and FCNs.

\subsection{Identifying Critical Samples}\label{subsec:idencritical}
We consider a two-dimensional space with all points between (0,0) and (1,1); randomly generated walls result in narrow corridors (Fig \ref{fig:bnresultsfig}).
We cover this space with a $r$-disc Halton graph with vertices in $\mathbb{R}^2$, each equipped with a feature vector. A constant graph with 2000 vertices is dense enough to cover any planning problem in this space.

%

\begin{figure*}[t]
  \centering
  \includegraphics[height=2.9cm,width=\linewidth]{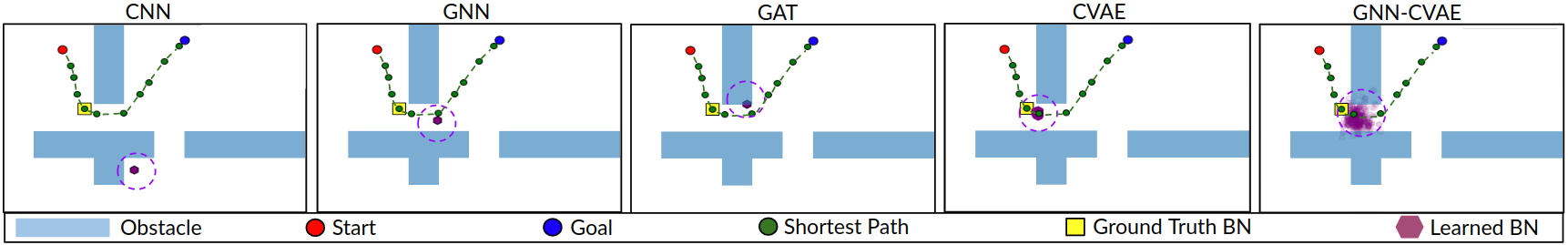}
  \includegraphics[height=2.8cm,width=\linewidth]{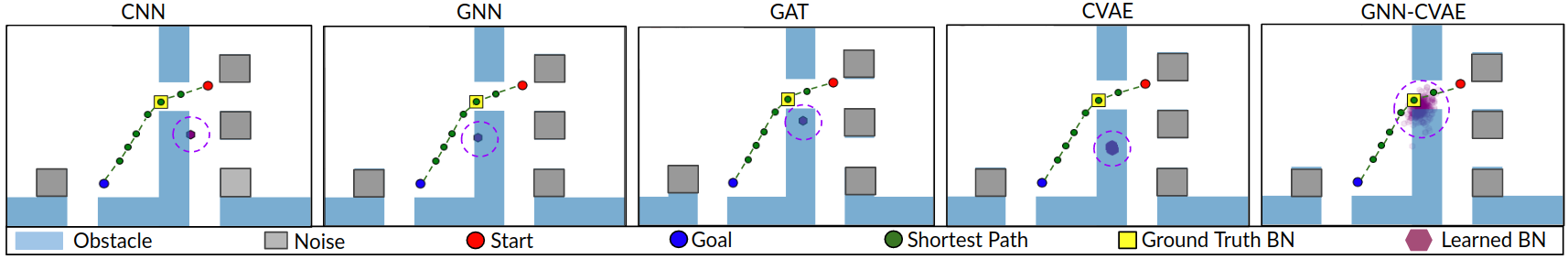}
  \caption{\textbf{Learning Critical Samples.} 
  \textit{Top row:} 
  GNNs and GATs predict critical nodes much more accurately than CNNs. CVAE predicts a distribution concentrated at one point, whereas GNN-CVAE represents the passageway more accurately.
  \textit{Bottom row:} On maps corrupted with random blobs unseen during training, only GNN-CVAE predicts critical nodes well.
  \label{fig:bnresultsfig}}

\end{figure*}

We computed bottleneck nodes for 400 motion planning problems with randomly generated maps, starts and goals. 
For a given motion planning problem $\mathbf{M}$, if the ground truth algorithm generates more than one bottleneck node, each can serve as a training label for $\mathbf{M}$. The resulting dataset is 15000 samples, split into train/validate/test sets. 

Fig \ref{fig:bnresultsfig}, top row, illustrates our qualitative findings. In most cases CNNs transfer poorly to environments significantly different from the training set. GNNs and GATs can handle multiple wall scenarios, producing critical nodes close to the ground truth. The CVAE tends to be more robust than the CNN, but it predicts a tight distribution of critical nodes concentrated around a single point. On the other hand GNN-CVAE paints a richer picture of the passageway: most of its probability mass lies close to the bottleneck node (Fig \ref{fig:bnresultsfig}).

We initially hypothesized planning on graphs with GNNs could outperform CNN/FCN architectures since the graph representation captures a richer picture of the planning problem and would be more robust to changes in the planning space. For example, a single pixel change can  fool CNNs \cite{su2019one}, whereas planning algorithms for realistic environments should be robust to small environment perturbations.  

\begin{figure}[hbt!]
    \includegraphics[width=\linewidth]{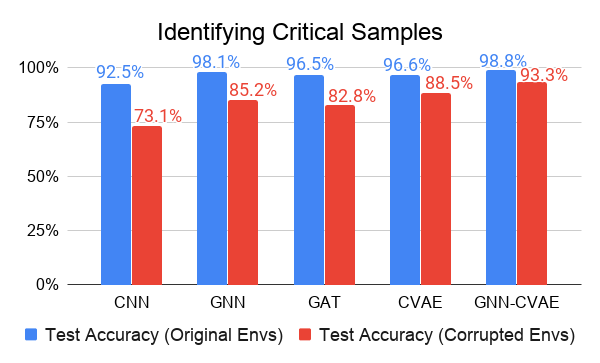}
  \caption{\textbf{Identifying Critical Samples.} Test accuracy averaged over 300 distinct planning problems.}\label{fig:bnacc}
\end{figure}

To evaluate robustness, we test on environments corrupted by random blobs not present in training (Fig \ref{fig:bnresultsfig}, bottom row). 
Only the GNN-CVAE can accurately predict a distribution centered around the ground truth node.
Fig \ref{fig:bnacc} shows accuracy as the complement of the mean squared error between predicted $\hat{\bby}$ and ground truth $\mathbf{y}$. 
We conclude
GNNs and GNN-CVAEs offer qualitative and quantitative advantages over CNN/FCN architectures for identifying critical nodes. 


\section{Learning Sampling Distributions with GNNs}
In Sec \ref{subsec:idencritical}
we operate under the assumption that the planning space can be discretized and be covered by a dense graph. 
The discretization assumption breaks for robots with kinodynamic constraintsm and the graph $\mathcal{G}$ covering the planning space $\mathcal{X}\in \mathbb{R}^d$ can require a very large number of nodes for dense coverage if d increases by even one or two dimensions.
To overcome these challenges, we adapt our GNN solutions to Sampling-based Planners (SBPs) such as Rapidly-exploring Random Trees (RRTs) which build an online graph by directly (uniformly) sampling the C-space. To study a problem with kinodynamic constraints that is non-trivial to discretize, we move away from 2D planar navigation to consider a pendulum and a six DoF robot arm.
\subsection{Pendulum Task}
The pendulum starts at the bottom with a given angle and velocity as goal. The pendulum is control limited and must plan a path that increases energy until the goal is reached. System state is $x=[\theta_p,\omega]$ where $\theta_p$ is the angle of the pendulum with the horizontal and $\omega$ is the angular velocity about the pendulum's center of mass. If samples $[x_1,x_2,\ldots] \sim \mu$ in the SBP are drawn from a distribution $\mu$ and the search algorithm $\mathbf{A}(\mathbf{M},x_1,x_2,\ldots)$ is extended to be a function of $\mathbf{M}$ as well as the samples ($x_1,x_2,\ldots)$, then the problem considered in this section is: 
\begin{problem}
\label{prob:learningdbn}
For planning problem $\mathbf{M}$ and SBP search algorithm $\mathbf{A}$ which samples states $[x_1,x_2,\ldots]$ from a distribution parameterized by $\mu$, select $\mu$ such that the cost $c$ of running
$\mathbf{A}(\mathbf{M},x_1,x_2,\ldots)$ is minimized:
\begin{equation}
\mu^* = \argmin_{\mu} \mathbb{E}_{\{\mathbf{M},\mu\}} \Big[c(\mathbf{A}(\mathbf{M},x_1,x_2,\ldots))\Big]
\end{equation}
\end{problem}

\begin{algorithm}[h!]  %
\begin{algorithmic}[1]
\State \textbf{Offline:}
\State Initialize empty dataset D = \{\}
\For {time $t=[0,\ldots,T]$} 
\State Initialize planning problem $\mathbf{M} =\{\mathcal{X}_{\text{free}},x_{\text{init}},x_{\text{goal}}\}$
\State Use uniform sampler in RRT to compute path $\mathbf{p}$ 
\State Initialize graph $\mathcal{G} = x_{\text{init}}$
\For {$p$ in range (0,length($\mathbf{p}$)-1)}
\State Update graph $\mathcal{G}: =x_{\text{init}} \oplus \mathbf{p}[p] $
\State Compute $\mathbf{x}$ for all nodes currently in $\mathcal{G}$.
\State Generate label $\mathbf{y}=\mathbf{p}[p+1]$
\State {Add to dataset D := \{$\mathcal{G},\mathbf{x},\mathbf{y}$\}}
\EndFor
\EndFor
\State Train parameters of sampling distribution $\mu(\mathcal{G},\mathbf{x})$.
\State \textbf{Online}:
\State Initialize new planning problem $\mathbf{M}$. 
\State Randomly sample points $s$ near $x_{\text{init}}$
\State Construct initial graph $\mathcal{G} := x_{\text{init}} \oplus s$
\State  Generate $\mathbf{x}$ for all nodes in current graph
\State  Predict next node $\hat{\mathbf{y}}$ to expand from $\mu(\mathcal{G},\mathbf{x})$ 
\While {$\hat{\mathbf{y}} \neq x_{\text{goal}}$}
\State Find node in $\mathcal{G}$ closest to $\hat{\mathbf{y}}$ and generate edge $e$ by integrating through dynamics of system 
\State Update graph $\mathcal{G} := \mathcal{G} \oplus
\mathbf{\hat{y}} $
\State Recompute $\mathbf{x}$ for updated graph
\State Predict new $\hat{\mathbf{y}}$ from $\mu(\mathcal{G},\mathbf{x})$
\EndWhile
\end{algorithmic}
\caption{Learning Sampling Distributions with GNNs.}\label{algo_critical}
\end{algorithm}

For a given $\mathbf{M}$, the graph $\mathcal{G}$ is initialized with the start node $x_{init}$ and a collection of $m$-nearest nodes connected to the start. Nodes are added to the graph from the sampler $\mu$. In this section, the graph $\mathcal{G}$ is \textbf{\textit{not constant}}
and is instead constructed as the RRT tree expands. This  increases the difficulty of the learning problem because the GNNs must learn to infer the topology of the problem through a time varying graph. Node features $n$ are given as $\mathbf{x}_n=[\Delta x_{init},\Delta x_{goal}]$ where $\Delta x_{init}= x_init - x_n$ and $\Delta x_{goal}= x_{goal} - x_n$ where $x_n$ represents node $n$ in configuration space.

For this problem, we only consider GNN, CVAE and GNN-CVAE architectures since there is no meaningful 2D representation for a CNN. Algorithm \ref{algo_critical} describes dataset collection, model training, and inference with the model.  Fig \ref{fig:rrtfig} shows generative architectures CVAE/GNN-CVAE offer a smaller benefit for this online planning case compared to GNNs. All three architectures improve upon uniform sampling by expanding fewer nodes and edges to produce similar solutions. For SBPs on the pendulum task, learning sampling distributions with GNNs offers substantial benefits over uniform samplers and FCN architectures.

\begin{figure}[t]
  \includegraphics[width=\linewidth]{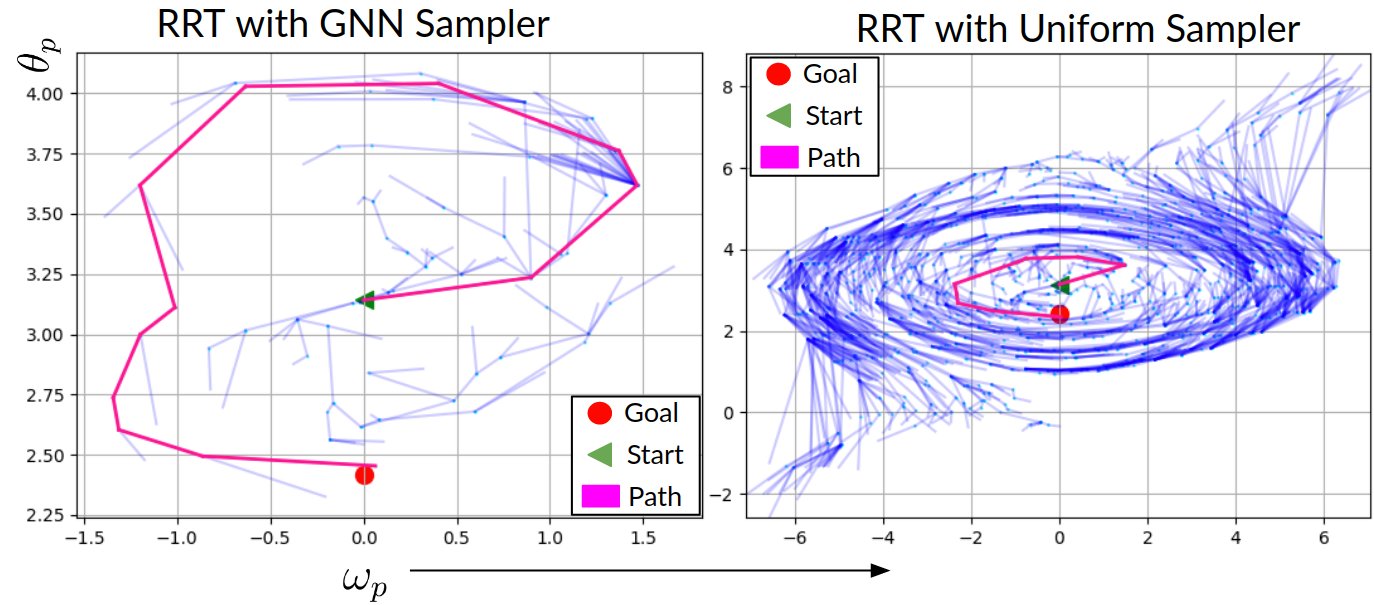}  
  \includegraphics[width=\linewidth]{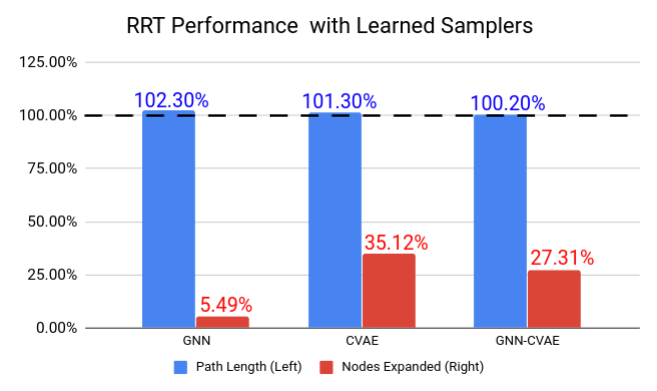}  
  \caption{\textbf{RRT with Learned Samplers for Pendulum} a) The GNN sampler expands fewer nodes than uniform sampling. b) Learned samplers show similar path lengths to uniform sampling but fewer expanded nodes.}
  \label{fig:rrtfig}
\end{figure}

\begin{figure}[t]
  \includegraphics[width=\linewidth]{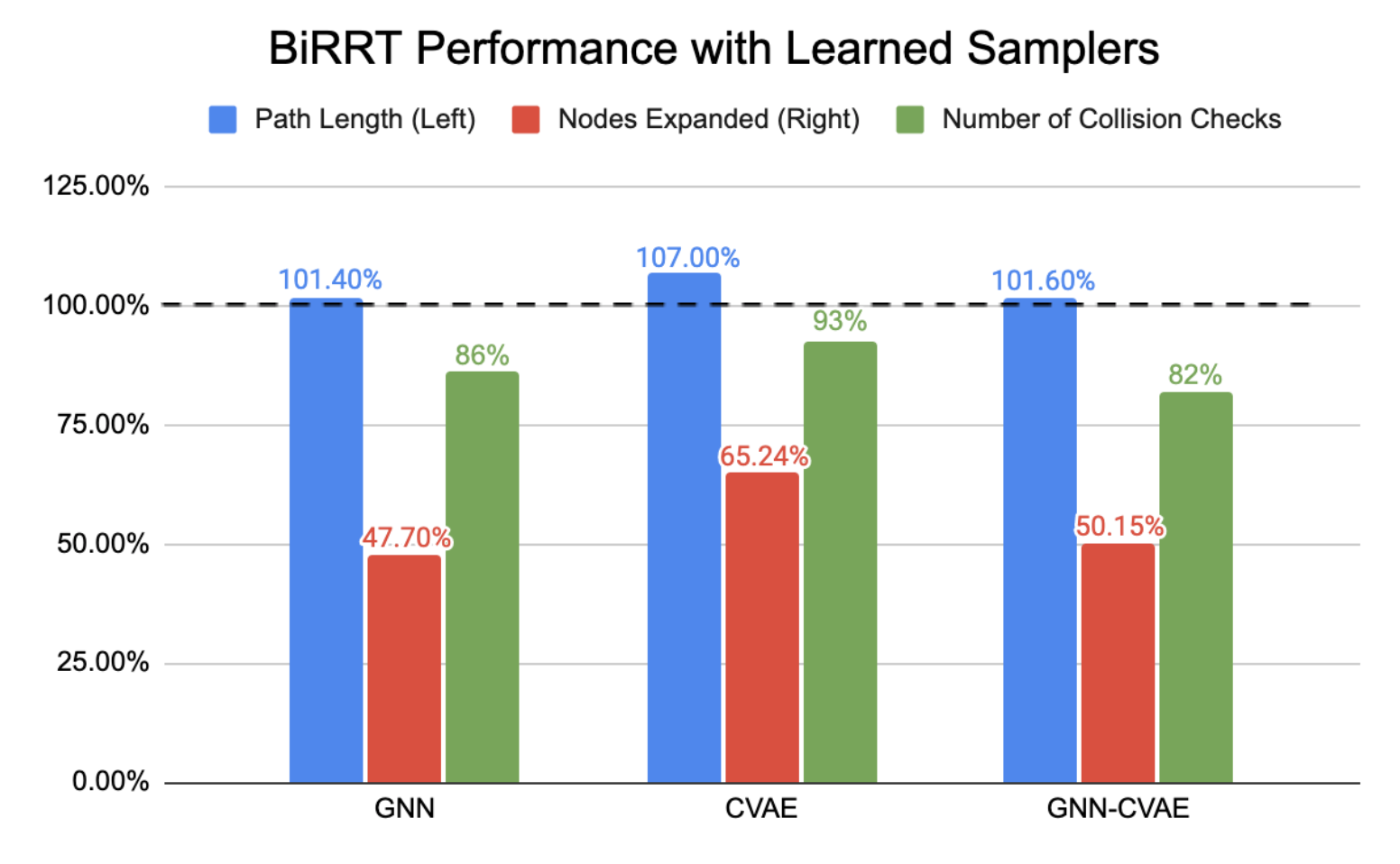}
  \caption{\textbf{Learned samplers for 6 DoF robot arm.} }
  \label{fig:birrt}
\end{figure}

\subsection{6 Degrees of Freedom Arm}
To validate GNN samplers for higher-dimensional planning, we test Algorithm \ref{algo_critical} on a simulated robot arm with six degrees of freedom which attempts to reach goals on tabletop environments with crevices and obstacles that block the arm's path. Vanilla RRT plans too slowly for this task, so we instead use Bi-directional RRT \cite{kuffner2000rrt} (BiRRT), which constructs two trees simultaneously from $x_{init}$ and $x_{goal}$, using a greedy heuristic which tries to connect the two trees to find a path. In our learned solution, we replace the uniform sampler in the tree growth step with our GNN sampler; features $\mathbf{x}$ for each node are constructed as before. We train on two environments with different numbers and types of obstacles present, then test on a third distinct environment. Figure \ref{fig:birrt} shows all learned architectures improve upon vanilla BiRRT in nodes expanded as well as collision checks. The GNN and GNN CVAE beat the CVAE on all fronts. 


\section{Conclusion}
GNNs offer a natural way to express functions in graph-based planning and can be adapted to many motion planning problems with ease. Three GNN-based architectures are proposed for motion planning problems, GNN, GAT and GNN-CVAE, using a densely sampled static graph for low-dimensional problems and a dynamically sampled graph for high-dimensional problems with kinodynamic constraints. GNN-CVAEs substantially outperforms CNNs for learning critical samples while GNNs outperform more powerful generative models that use fully-connected networks for high-dimensional sampling-based planning.

While adapting learning methods to very high-dimensional problems remains difficult, we see many ways to extend these methods. State vectors for robots can be replaced by sensor readings such as lidar or camera images. Online sampling distribution can be adapted for long-range navigation by tree pruning. We leave these for future work.




\bibliographystyle{IEEEtran}
\bibliography{root}

\end{document}